\documentclass[journal]{IEEEtran}
\usepackage{array}
\usepackage{textcomp}
\usepackage{stfloats}
\usepackage{url}
\usepackage{verbatim}
\usepackage{graphicx}

\usepackage{cite}
\usepackage{amsmath,amssymb,amsfonts}
\usepackage{algorithm,algorithmic}
\usepackage{textcomp}
\usepackage{xcolor}
\usepackage{siunitx}
\usepackage{upgreek}
\usepackage{fancyhdr}
\usepackage{makecell}
\usepackage{marvosym}
\usepackage{amsthm} 
\usepackage{booktabs}
\usepackage{multirow}

\newtheorem{assumption}{\bf Assumption}
\newtheorem{definition}{\bf Definition}
\newtheorem{theorem}{\bf Theorem}

\newtheorem{corollary}{\bf Corollary}

\newcommand{\revise}[1]{#1}

\usepackage{hyperref}  

\begin{document}

\title{Preference-based Multi-Objective \\ Reinforcement Learning}


\author{Ni Mu$^*$, Yao Luan$^*$, Qing-Shan Jia$^\dagger$ 
\thanks{N. Mu, Y. Luan and Q. Jia are with the Center for Intelligent and Networked System (CFINS), Department of Automation, Beijing National Research Center for Information Science and Technology, Beijing Key Laboratory of Embodied Intelligence Systems,  Tsinghua University, Beijing 100084, China, {\tt\small \{mn23@mails., luany23@mails., jiaqs@\} tsinghua.edu.cn}. 
$^*$N. Mu and Y. Luan contributed equally. $^\dagger$Q. Jia is the corresponding author. }
}

\markboth{Journal of \LaTeX\ Class Files,~Vol.~14, No.~8, August~2021}%
{Shell \MakeLowercase{\textit{et al.}}: A Sample Article Using IEEEtran.cls for IEEE Journals}


\maketitle

\begin{abstract}

Multi-objective reinforcement learning (MORL) is a structured approach for optimizing tasks with multiple objectives.
However, it often relies on pre-defined reward functions, which can be hard to design for balancing conflicting goals and may lead to oversimplification.
Preferences can serve as more flexible and intuitive decision-making guidance, eliminating the need for complicated reward design.
This paper introduces preference-based MORL (Pb-MORL), which formalizes the integration of preferences into the MORL framework. 
We theoretically prove that preferences can derive policies across the entire Pareto frontier.
To guide policy optimization using preferences, our method constructs a multi-objective reward model that aligns with the given preferences. 
We further provide theoretical proof to show that optimizing this reward model is equivalent to training the Pareto optimal policy.
Extensive experiments in benchmark multi-objective tasks, a multi-energy management task, \revise{and an autonomous driving task on a multi-line highway} show that our method performs competitively, surpassing the oracle method, which uses the ground truth reward function. This highlights its potential for practical applications in complex real-world systems.

\end{abstract}

\def\abstractname{Note to Practitioners}
\begin{abstract}

Decision-making problems with multiple conflicting objectives are common in real-world applications, e.g., 
energy management must balance system lifespan, charge-discharge cycles, and energy procurement costs; 
autonomous driving vehicles must balance safety, speed, and passenger comfort. 
While multi-objective reinforcement learning (MORL) is an effective framework for these problems, its dependence on pre-defined reward functions can limit its application in complex situations, as designing a reward function often fails to capture the full complexity of the task fully.
This paper introduces preference-based MORL (Pb-MORL), which utilizes user preference data to optimize policies, thereby eliminating the complexity of reward design. 
Specifically, we construct a multi-objective reward model that aligns with user preferences and demonstrate that optimizing this model can derive Pareto optimal solutions.
Pb-MORL is effective, easy to deploy, and is expected to be applied in complex systems, e.g., multi-energy management through preference feedback \revise{and adaptive autonomous driving policies for diverse situations}.

\end{abstract}


\section{Introduction} \label{sec:intro}

Multi-objective optimization is pervasive in real-world applications \cite{multi-objective-realworld, LIU2023119521, guan2006constrained}.
For example, in an energy system, the goal is to maximize the system lifespan and minimize charge-discharge cycles while simultaneously reducing energy procurement costs \cite{baghaee2016reliability}. 
Autonomous vehicles need to provide safe, fast, and comfortable rides at the same time \cite{MO_auto_driving}.
However, representing these objectives with a single reward can be difficult and may lose important information \cite{Bpref, memarian2021self, mannion2018reward}. 
In addition, creating a scalar reward function for each control objective is challenging and often results in oversimplification \cite{pebble, dewey2014reinforcement}. Preferences, conversely, offer a more flexible and general way to model the decision-making process \cite{pbrl_basic2017}. Humans can easily provide their preferences, pointing out which outcome they prefer, without compressing all their decision-making information into a single reward function \cite{pbrl_basic2012}. 
Therefore, it is of great practical interest to study multi-objective reinforcement learning based on preference.

However, integrating multi-objective reinforcement learning (MORL) with preference-based learning presents several challenges. 
First, while users can express preferences between pairs of behaviors when focusing on a single objective, establishing a complete ordering among all behaviors is often difficult. This lack of a complete preference makes it hard for algorithms to assess the relative importance of different objectives. 
Additionally, there are often inherent conflicts between objectives, which complicate policy optimization.
Furthermore, obtaining preferences for all objectives may require pairwise comparisons, which can be computationally inefficient as the number of objectives increases, leading to increased complexity in the querying process.
Given these complexities, a significant gap lies in the previous work: 
To the best knowledge of the authors, we did not find any method addressing the above challenges of combining MORL and preference-based optimization, highlighting the need for a novel approach to multi-objective, preference-driven decision-making problems.

\begin{figure*}[t!]
\centering
\includegraphics[width=0.8\linewidth]{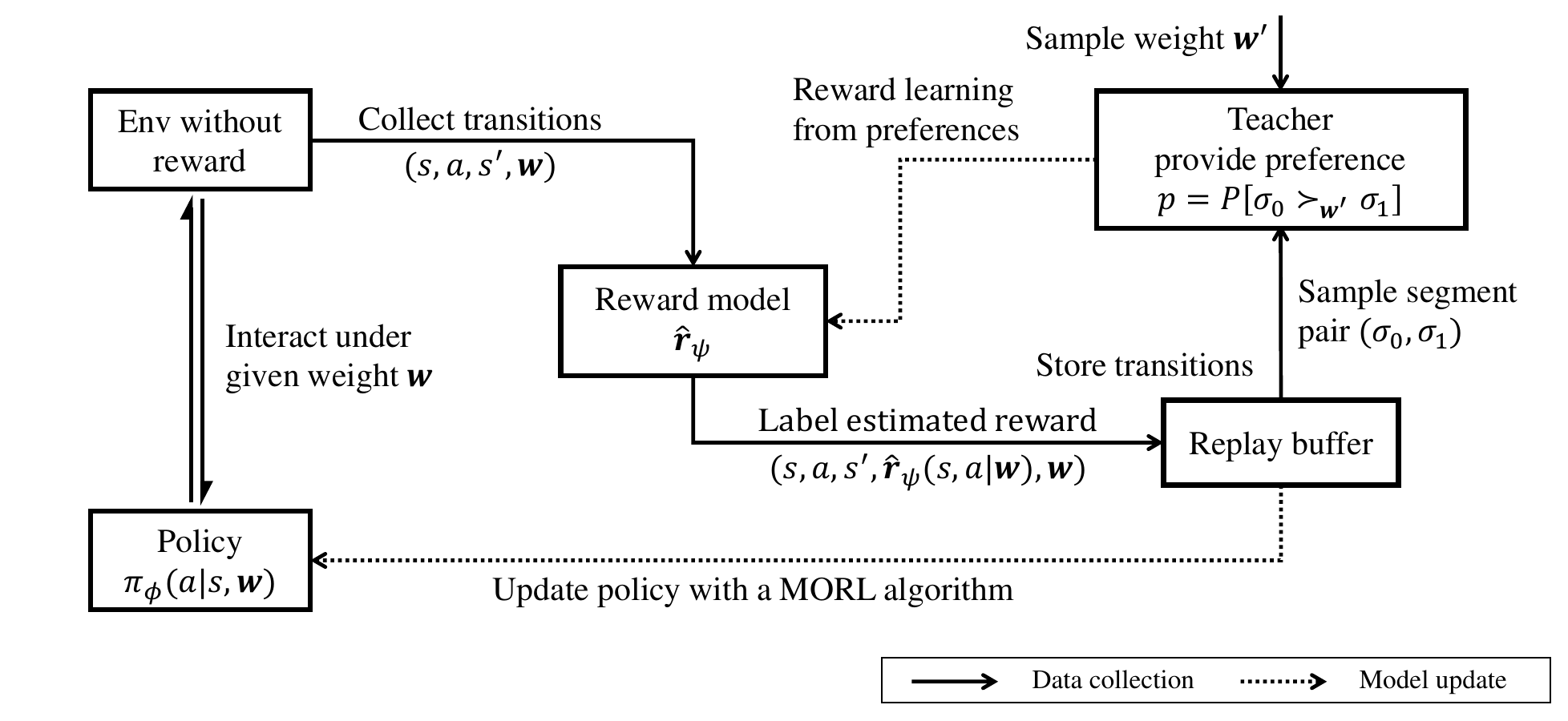}
\caption{
A demonstration of the proposed Pb-MORL framework. 
An explicit multi-objective reward model $\hat{\boldsymbol{r}}_\psi$ is learned using preference data. 
Then, the multi-objective policy $\pi_\phi(a|s,\boldsymbol{w})$ can be updated through any MORL algorithm based on the reward model. 
In this figure, $\boldsymbol{w}$ denotes the weight vector,
$(\sigma_0, \sigma_1)$ denotes the segment pair,
$p$ denotes the preference provided by the teacher.
A detailed introduction to the settings and notations will be provided in the following sections.
}
\label{fig:method_architecture}
\end{figure*}

\subsection{Related Work}

\textbf{Single-objective reinforcement learning with preference.}
Reinforcement learning (RL) \cite{RL_sutton} has gained significant attention in recent years due to its remarkable success in solving challenging problems \cite{2013_RL_atari, EMAPP, 2017_RL_alphazero, luanyao_cac2023, muni_case2024, luan2025efficient}.
Traditional RL algorithms often rely on a pre-defined reward function, which serves as the guidance for policy optimization. However, designing such a reward function can be complex and even impractical \cite{pebble}. 
However, there are two typical categories of RL problems where we face difficulty obtaining the optimal policy. 
The first category of RL problems has a pre-defined reward function, such as cost savings \cite{muni_cac2023}, reducing carbon emissions \cite{guo2024incentive}, or sparse rewards like a ``$+1$'' bonus for reaching the goal in a maze \cite{openai_gym}. 
The challenge in this category is identifying the optimal policy, as task dynamics are often complex and stochastic, while sparse reward signals complicate policy learning from the reward function \cite{RND}.
The second category consists of RL problems where defining the reward function is challenging, such as in robotics systems \cite{pebble} and large language models \cite{DPO}. In these cases, while we want the system behaviors to align with human expectations, formalizing the objective function is often difficult \cite{pebble}.
Preference-based reinforcement learning (PbRL) provides a solution by utilizing user feedback to guide agent behavior, making it suitable for both categories of RL problems. This approach offers preferences that may be more accessible and more naturally aligned with policy optimization than traditional reward signals. 
Early works of PbRL, such as \cite{pbrl_basic2017} and \cite{pbrl_basic2012}, have shown the ability of agents to learn from simple comparisons between pairs of trajectory segments, thereby eliminating the necessity for complex reward engineering.
With the development of deep learning, techniques like pre-training \cite{ibarz2018reward, Bpref, mu2024sepoa} and data augmentation \cite{surf} are employed to improve learning efficiency. Meta-learning approaches \cite{hejna2023few} also enable agents to adapt to new tasks based on past experiences quickly.
Moreover, PbRL has been successfully applied to fine-tune large-scale language models like GPT-3 for challenging tasks, as highlighted by \cite{GPT3_RLHF(?)}.
While PbRL omits reward engineering through leveraging user feedback, it primarily deals with single-objective optimization instead of multi-objective preference modeling.

\textbf{Multi-objective reinforcement learning with explicit reward functions.} 
Multi-objective reinforcement learning (MORL) is a pivotal subfield of reinforcement learning \cite{MO-gym, hayes2022practical, li2024human}, focusing on decision-making problems under multiple objectives.
Envelope Multi-objective Q-learning \cite{EQL} extends the traditional Q-learning algorithm to the multi-objective domain and proves the convergence of its Q-learning algorithm in tabular settings. 
Expected Utility Policy Gradient (EUPG) \cite{EUPG} and Prediction-Guided MORL (PGMORL) \cite{PGMORL} further integrate deep learning into MORL. EUPG incorporates policy gradients to balance current and prospective returns, while PGMORL applies an evolutionary strategy to enhance the Pareto frontier.
Additionally, Pareto Conditioned Networks \cite{pareto_conditioned_networks} and Generalized Policy Improvement Linear Support \cite{GPI-LS} employ neural networks conditioned on target returns to predict optimal actions within deterministic settings.
Despite their advancements, current MORL methods rely on pre-defined multi-objective reward functions, posing challenges for their application in real-world control scenarios.
Extending preferences from single-objective reinforcement learning to multi-objective contexts is feasible, which is the main contribution of this paper.

\begin{table*}[t]
\centering
\caption{Notation Table}
\label{tab:notation}
\begin{tabular}{lll}
\toprule
Symbol & Definition & Description \\
\midrule

\multicolumn{3}{l}{\textit{Multi-Objective RL Elements}} \\
[1ex]

$m$ & 
Number of objectives & 
Dimension of the reward vector. \\
[1ex]

$\boldsymbol{r}(s, a)$ & 
True multi-objective reward vector & 
$\boldsymbol{r}(s, a) \in \mathbb{R}^m$: Ground truth reward signal provided by the environment for each objective. \\
[1ex]

$\boldsymbol{w}$ & 
Weight vector & 
$\boldsymbol{w} \in \mathcal{W} = \{\boldsymbol{w} \in \mathbb{R}^m \mid w_i \geq 0, \sum_i w_i = 1\}$: Vector encoding the relative importance \\
& & (preference) assigned to each objective. \\
[1ex]

$\mathcal{W}$ & 
Weight space & 
Set of all valid weight vectors. \\
[1ex]

$D_w$ & 
Prior weight distribution & 
Distribution over the weight space $\mathcal{W}$ from which weights are sampled during training or \\
& & evaluation. In this study, we assume this distribution to be uniform over $\mathcal{W}$. \\
[1ex]

$\Pi$ & 
Policy space & 
Set of all possible policies. \\
[1ex]

$\Pi^*$ & 
Pareto optimal policy set & 
Set of policies that are not dominated by any other policy in $\Pi$ with respect to all objectives. \\

\midrule
\multicolumn{3}{l}{\textit{Preference Elements}} \\
[1ex]

$\sigma$ & 
Trajectory segment & 
Finite sequence of state-action pairs: $\sigma = \{s_k, a_k, \dots, s_{k+H-1}, a_{k+H-1}\}$ of length $H$. \\
[1ex]

$H$ & 
Segment length & 
Number of steps in a trajectory segment $\sigma$. \\
[1ex]

$p$ & 
Preference label & 
$p \in \{0, 0.5, 1\}$: Human teacher's preference judgment for a pair of segments $(\sigma_0, \sigma_1)$ under \\
& & weight $\boldsymbol{w}$. $p=0$: $\sigma_0$ preferred, $p=1$: $\sigma_1$ preferred, $p=0.5$: no preference/indifferent. \\
[1ex]

$\sigma_0 \succ_{\boldsymbol{w}} \sigma_1$ & 
Preference relation & 
Segment $\sigma_0$ is strictly preferred over segment $\sigma_1$ under weight $\boldsymbol{w}$. \\

\midrule
\multicolumn{3}{l}{\textit{Learned Components}} \\
[1ex]

$\hat{\boldsymbol{r}}_{\psi}(s, a)$ & 
Learned multi-objective reward model & 
$\hat{\boldsymbol{r}}_{\psi}(s, a) \in \mathbb{R}^m$: Model parameterized by $\psi$, trained using preference data to approximate \\
& & the underlying objectives. Used as the reward signal for MORL policy optimization. \\
[1ex]

$\pi_{\phi}(a\|s, \boldsymbol{w})$ & 
Parameterized policy & 
Stochastic policy parameterized by $\phi$, conditioned on the current state $s$ and the weight \\
& & vector $\boldsymbol{w}$ (indicating the desired objective trade-off). Outputs a distribution over actions. \\
[1ex]

$Q_{\theta}(s, a, \boldsymbol{w})$ & 
Multi-objective Q-function & 
$Q_{\theta}(s, a, \boldsymbol{w}) \in \mathbb{R}^m$: State-action value function parameterized by $\theta$. Estimates the vector of \\
& & expected discounted future rewards for each objective, starting from state $s$, taking action $a$, \\
& & and following policy $\pi_{\phi}(\cdot\| \cdot, \boldsymbol{w})$ thereafter. \\
[1ex]

$J$ & 
Optimization objective & 
Scalarized expected return: $J = \mathbb{E}_{\boldsymbol{w} \sim D_w, \tau \sim (\mathcal{P},\pi)} [ \boldsymbol{w}^T \sum_t \gamma^t \hat{\boldsymbol{r}}_{\psi}(s_t, a_t) ]$ (Eq. 7). \\
& & Maximized during policy learning using the learned reward model. \\

\bottomrule
\end{tabular}
\end{table*}

\subsection{Main Contributions}

In this paper, we introduce Preference-based Multi-Objective Reinforcement Learning (Pb-MORL), which integrates preference modeling into Multi-Objective Reinforcement Learning (MORL), as illustrated in Fig. \ref{fig:method_architecture}.
Specifically, we first establish theorems that demonstrate a teacher providing preferences can guide the learning of optimal multi-objective policies (Theorem \ref{thm:1}, \ref{thm:2}).
Furthermore, we propose a method to construct an explicit multi-objective reward model that aligns with the teacher's preferences. Our theoretical proof (Theorem \ref{thm:4}) shows that, in a multi-objective context, if the reward function perfectly matches the teacher's preferences, optimizing this reward is equivalent to learning the optimal policy.
To implement Pb-MORL, we combine the Envelope Q Learning (EQL) method \cite{EQL} with our proposed reward model. This implementation is simple yet effective, for EQL guarantees the convergence of policy optimization in multi-objective tasks.
To demonstrate the effectiveness of our method, we conduct experiments in benchmark multi-objective reinforcement learning tasks. The results show that our approach achieves performance levels comparable to the oracle method, which uses the ground truth reward function to learn the optimal policy. 
To validate our method's applicability in real-world scenarios, we evaluate our method on \revise{both a multi-energy management task and an autonomous driving task on a multi-line highway}. In both settings, the Pb-MORL algorithm outperforms the oracle, showing its potential for practical implementation in complex, real-world environments.
Through this work, we aim to broaden the applications of MORL in real-world settings, by employing preferences as a more accessible and intuitive optimization guidance.

The main contributions of this paper are as follows:

\begin{itemize}
\item We establish theorems for preference-based optimization in multi-objective settings, demonstrating that a preference-based teacher can guide the learning of optimal multi-objective policies (Theorem \ref{thm:1}, \ref{thm:2}, \ref{thm:4}).
\item We introduce Pb-MORL, which develops an explicit multi-objective reward model that aligns with preference data through the construction of the Bradley-Terry model and the optimization of the cross-entropy loss function. In addition, we combine the EQL algorithm with the reward model to achieve a simple yet effective implementation of Pb-MORL.
\item We conduct experiments in multi-objective benchmark tasks, a multi-energy management task, \revise{and an autonomous driving task on a multi-line highway}, showing that Pb-MORL performs comparably to the oracle method using ground truth reward functions. It demonstrates Pb-MORL's potential for real-world applications.
\end{itemize}

The remaining sections are organized as follows.
In Section \ref{sec:pblm_fmlation}, we introduce preliminaries and the problem formulation.
In Section \ref{sec:method}, we present the theoretical guarantees of Pb-MORL and propose the specific algorithm for explicit reward modeling and policy optimization. 
In Section \ref{sec:experiment}, we describe the experimental setting and discuss the experimental results.
Finally, we conclude the paper in Section \ref{sec:conclusion}.

\section{Problem Formulation} \label{sec:pblm_fmlation}

In this section, we first introduce the multi-objective MDP and Q-learning in multi-objective settings, then formulate the Pb-MORL framework.

\subsection{MDP and Q-learning in Multi-Objective Settings}

For single-objective settings, an MDP with discrete time, infinite-stage discounted reward, and finite or countable state and action spaces could be characterized as a tuple $\mathcal{M} = \langle \mathcal{S}, \mathcal{A}, P, r, \gamma \rangle$.
Here, $\mathcal{S}$ is the state space, $\mathcal{A}$ is the action space, and $P(s'|s, a): \mathcal{S} \times \mathcal{A} \times \mathcal{S} \rightarrow [0, 1]$ is the one-step state transition probability of transiting from $s$ to $s'$ by taking action $a$. Besides, $r(s, a): \mathcal{S} \times \mathcal{A} \rightarrow \mathbb{R}$ defines the immediate reward of taking action $a$ under state $s$, and $\mathbb{R}$ denotes the set of real numbers. Finally, $\gamma \in (0, 1)$ is the discount factor for balancing immediate and long-term rewards.

For multi-objective settings, the MDP framework is extended to include multiple reward functions. The reward function can be represented as a vector $\boldsymbol{r}(s, a): \mathcal{S} \times \mathcal{A} \rightarrow \mathbb{R}^{m}$, where $m$ is the number of objectives. In the case of linear reward combination, the overall reward is defined by a linear combination of these objectives, $r_{\boldsymbol{w}}(s, a) = \boldsymbol{w}^T \boldsymbol{r}(s, a)$, where $\boldsymbol{w} \in \mathcal{W}$ is the weight vector, and the weight space $\mathcal{W}=\{\boldsymbol{w} | \boldsymbol{w}\in \mathbb{R}^{m}, w_i\ge 0, \sum w_i = 1\}$.
The goal in the multi-objective MDP is to find a policy \revise{$\pi(a|s,\boldsymbol{w}): \mathcal{S} \times \mathcal{W} \times \mathcal{A} \rightarrow [0,1]$} that maximizes the inner product of the multi-dimensional discounted return and the weight vector $\boldsymbol{w}$, that is,
\begin{equation}
    \label{eq:PbMORL_objective}
    \max J = \mathbb{E}_{\substack{\boldsymbol{w}\sim D_{\boldsymbol{w}} \\ 
    \tau\sim \left(P, \revise{\pi(\cdot|\cdot, \boldsymbol{w})} \right)}}
    \boldsymbol{w}^T \sum_\tau \gamma^t\boldsymbol{r}(s_t,a_t), 
\end{equation}
where $\tau$ denotes the trajectory, and under $D_{\boldsymbol{w}}$ is a prior weight distribution.
Denote the policy space as $\Pi$. 

Then, the Q-learning algorithm can be adapted to the multi-objective setting.
The standard Q-Learning \cite{Q_learning, RL_sutton} for single-objective RL is based on the Bellman optimality operator $\mathcal{B}$: 
\begin{equation}
    (\mathcal{B}Q)(s,a) := r(s,a) + \sup_{a'}\gamma\mathbb{E}_{s'\sim P(\cdot|s,a)}Q(s',a').
\end{equation}
Following the previous work \cite{EQL}, we extend this to the MORL setting, by considering multi-objective Q-value functions $\boldsymbol{Q}(s, a, \boldsymbol{w}): \mathcal{S}\times \mathcal{A} \times \mathcal{W} \rightarrow \mathbb{R}^{m}$, which estimates expected total vector rewards under state $s$, action $a$ and $m$-dimensional weight $\boldsymbol{w}$. 
\revise{It is important to note that the parameter $\boldsymbol{w}$ in the Q function represents that the Q value is under the policy \revise{$\pi(\cdot | \cdot, \boldsymbol{w})$}, because the policies conditioning on different weight vectors $\boldsymbol{w}$ results in different behaviors, and the corresponding Q functions can vary. }

\revise{We define the distance between two multi-objective Q functions $\boldsymbol{Q}_1, \boldsymbol{Q}_2$ as follows: }
\begin{equation}
    \label{eq:metric}
    d(\boldsymbol{Q}_1, \boldsymbol{Q}_2) := \sup_{s,a,\boldsymbol{w}}
    \left|\boldsymbol{w}^T \left( \boldsymbol{Q}_1(s, a, \boldsymbol{w}) - \boldsymbol{Q}_2(s, a, \boldsymbol{w}) \right)\right|.
\end{equation}
The metric $d$ forms a complete pseudo-metric space, as the identity of indiscernibles \cite{indiscernibles} does not hold. 

With a little abuse of notation, we use the same $\mathcal{B}_\pi$ and $\mathcal{B}$ as in the single-objective RL to represent the Bellman operator in the multi-objective setting. 
Specifically, given a policy $\pi$ and sampled trajectories $\tau$, the multi-objective Bellman operator for policy evaluation $\mathcal{B}_\pi$ is defined as:
\begin{equation}
    \label{eq:B_pi}
    (\mathcal{B}_\pi\boldsymbol{Q})(s,a,\boldsymbol{w}) :=
    \boldsymbol{r}(s,a) + \gamma\mathbb{E}_{\tau\sim (P,\pi)}
    \boldsymbol{Q}(s^{\prime},a^{\prime},\boldsymbol{w}).
\end{equation}
To construct the multi-objective Bellman optimality operator, an optimality filter $\mathcal{H}$ for the multi-objective Q function is first defined as:
\begin{equation}
(\mathcal{H} \boldsymbol{Q})(s, \boldsymbol{w}) := \arg_{\boldsymbol{Q}} 
\sup_{a\in\mathcal{A}, \boldsymbol{w}'\in \mathcal{W}} 
\boldsymbol{w}^T\boldsymbol{Q}(s, a, \boldsymbol{w}'),
\end{equation}
where the $\arg \boldsymbol{Q}$ takes the multi-objective value corresponding to the supremum (i.e., $\boldsymbol{Q}(s, a, \boldsymbol{w}')$ ) such that $(a, \boldsymbol{w}') \in \arg \sup_{a \in\mathcal{A}, \boldsymbol{w}'\in \mathcal{W}} \boldsymbol{w}^T \boldsymbol{Q}(s, a, \boldsymbol{w}')$).
Then, the multi-objective Bellman optimality operator $\mathcal{B}$ is defined as:
\begin{equation}
    \label{eq:B_optimal}
    \begin{aligned}
        (\mathcal{B}\boldsymbol{Q})(s,a,\boldsymbol{w}) := 
        \boldsymbol{r}(s,a) + \gamma 
        \revise{\mathbb{E}_{s'\sim P(\cdot|s,a)} (\mathcal{H} \boldsymbol{Q})(s', \boldsymbol{w})}
    \end{aligned}
\end{equation}
Intuitively, the optimality Bellman operator $\mathcal{B}$ solves the minimum convex envelope of the current $\boldsymbol{Q}$ frontier.
Previous works of MORL \cite{EQL} have provided proof of the convergence of the above multi-objective Q-learning algorithm, by proving the Bellman operator $\mathcal{B}_\pi$ and $\mathcal{B}$ are both contrastive mappings under the metric $d$ defined in Eq. \eqref{eq:metric}.

\subsection{Pb-MORL Formulation}
\label{subsec:formulation}

For single-objective settings, by following the previous work \cite{pbrl_basic2012,pbrl_basic2017}, we can define the preference in the form of tuple $(\sigma_0, \sigma_1, p)$, where segment $\sigma_0, \sigma_1$ are sequences of states and actions $\{s_k, a_k, ..., s_{k+H-1}, a_{k+H-1}\}$ with length $H$ and arbitrary starting time $k$, 
\revise{and $p\in \{0, 0.5, 1\}$ encodes the preference relations:
\begin{itemize}
    \item $\sigma_0$ strictly preferred to $\sigma_1$ when $p=0$.
    \item $\sigma_1$ strictly preferred to $\sigma_0$ when $p=1$. 
    \item Indeterminate preference (equivalence or ambiguous judgment) when $p=0.5$.
\end{itemize}
This scheme accounts for human rating uncertainty while maintaining annotation efficiency. When $\sigma_0 = \sigma_1$ or trajectories are equally preferable, $p=0.5$ explicitly captures the uncertainty. 
}

For multi-objective settings, we redefine the preference as a tuple $(\sigma_0, \sigma_1, \boldsymbol{w}, p)$, where $\boldsymbol{w}\in \mathcal{W}$ is a weight vector. 
\revise{The preference $p\in \{0,0.5,1\}$ is a scalar which encodes preference relations under $\boldsymbol{w}$, defined similarly as in the single-objective settings.}
In fact, given any weight vector, we introduce a complete ordering of the trajectory segments. However, we employ a pairwise comparison method due to practical constraints and use partial ordering notation ($\succ$) in the following paper.
Specifically, let $\sigma_0 \succ_{\boldsymbol{w}} \sigma_1$ means that trajectory segment $\sigma_0$ is preferred over $\sigma_1$ under the weight vector $\boldsymbol{w}$.
Then, the preference $p$ can be written in the form of $p=\mathbb{I}(\sigma_0 \succ_{\boldsymbol{w}} \sigma_1)$, where $\mathbb{I}(\cdot)$ is an indicator function that returns $1$ if the condition is true, and $0$ otherwise. 
As mentioned earlier, the weight $\boldsymbol{w}$ represents the importance assigned to each objective within the multi-objective framework. 
The weight $\boldsymbol{w}$ is crucial in defining the multi-objective preference, as preferences can vary for the same trajectory pair depending on the weights.

\begin{algorithm}[t]
\caption{Using the teacher to derive convex Pareto frontier, based on traversing the weight space}
\label{alg:1}
\begin{algorithmic}[1]
\STATE Initialize the solution set $\Pi^*=\emptyset$
\FOR{each $\boldsymbol{w}\in \mathcal{W}^{[N_w]}$}
    \FOR{each $\pi_i\in\Pi$}
        \IF{$\nexists ~\pi'\in \Pi, \pi'\neq\pi_i$ s.t. 
        $\boldsymbol{w}^T \sum_{(s,a)\sim\sigma_i} \gamma^t \boldsymbol{r}(s_t,a_t) < 
        \boldsymbol{w}^T \sum_{(s,a)\sim\sigma'} \gamma^t \boldsymbol{r}(s_t,a_t)$, where $\sigma_i, \sigma'$ are segments generated by $\pi_i, \pi'$, }
            \STATE $\Pi^*\gets\Pi^* \cup\{\pi_i\}$
        \ENDIF
    \ENDFOR
\ENDFOR
\RETURN $\Pi^*$
\end{algorithmic}
\end{algorithm}

To align the problem formulation with the RL framework, we define an explicit multi-objective reward model $\hat{\boldsymbol{r}}: \mathcal{S}\times \mathcal{A} \rightarrow \mathbb{R}^{m}$, where each dimension corresponds to a distinct objective. 
This reward model can be trained using preference data, serving as a bridge between qualitative preference and quantitative rewards.
Based on this model, we propose the objective of Pb-MORL as finding a policy \revise{$\pi(a|s,\boldsymbol{w})$} conditioned on the weight vector $\boldsymbol w$. Specifically, the goal is to maximize the inner product between the conditioned weight and the discounted return of the reward model $\hat{\boldsymbol{r}}$, under a prior weight distribution $D_{\boldsymbol{w}}$, as Eq. \eqref{eq:PbMORL_objective} shows:
\begin{equation}
    \max J = \mathbb{E}_{\substack{\boldsymbol{w}\sim D_{\boldsymbol{w}} \\ 
    \tau\sim \left(P, \revise{\pi(\cdot|\cdot, \boldsymbol{w})} \right)}}
    \boldsymbol{w}^T \sum_\tau \gamma^t\hat{\boldsymbol{r}}(s_t,a_t).
\end{equation}
In the form of Q function, it can also be written as:
\begin{equation}
    \max J = \mathbb{E}_{\substack{\boldsymbol{w}\sim D_{\boldsymbol{w}} \\ 
    (s_0,a_0) \sim \left(P, \revise{\pi(\cdot|\cdot, \boldsymbol{w}}) \right)}}
    \boldsymbol{w}^T \boldsymbol{Q}_\pi
    (s_0,a_0,\boldsymbol{w}|\hat{\boldsymbol{r}}).
\end{equation}

\section{A Pb-MORL algorithm with Explicit Reward Modeling} 
\label{sec:method}

\subsection{Theoretical Analysis}

In this subsection, we present the theoretical foundations of the Pb-MORL framework. 
We demonstrate how our approach ensures convergence to Pareto-optimal policies.
To ease the proof, we discretize the weight space $\mathcal{W}$ to a finite space $\mathcal{W}^{[N_w]}$ with size $N_w$, 
and assume that when $N_w$ is large enough, $\mathcal{W}^{[N_w]}$ could fully represent $\mathcal{W}$, then induce the same set of optimal policies.
We formalize this in Assumption \ref{asp:discrete_w}.

First, we assume the presence of preferences over pairs of trajectory segments with arbitrary finite length under an arbitrary given weight.
To formalize this, we introduce the following assumption.

\begin{assumption} \label{asp:1}
The preference $p\in \{0,0.5,1\}$ over a pair of trajectory segments $(\sigma_0, \sigma_1)$ exists, with arbitrary finite segment length $H$, under an arbitrary given weight $\boldsymbol{w}\in W$. 
These preferences satisfy \textbf{symmetry}, \textbf{consistency}, and \textbf{transitivity}, which are defined as follows.
\end{assumption}

\begin{definition}[Symmetry]  
Symmetry means that if trajectory segment $\sigma_0$ is preferred over $\sigma_1$ under a weight vector $\boldsymbol{w}$, then the opposite must also be true: $\sigma_1$ is less preferred than $\sigma_0$ under the same weight. Formally, this is written as:
\begin{equation}
\sigma_0 \succ_{\boldsymbol{w}} \sigma_1 \implies \sigma_1 \prec_{\boldsymbol{w}} \sigma_0 .
\end{equation}
This ensures that preferences are reversible under the same weight vector.
\end{definition}

\begin{definition}[Consistency]
Consistency means that if $\sigma_0 \succ_{\boldsymbol{w}} \sigma_1$ holds for a given $\boldsymbol{w}$, this preference remains unchanged over time.
Formally, this is expressed as: 
\begin{equation}
\sigma_0^{t_0} \succ_{\boldsymbol{w}} \sigma_1^{t_0} \implies 
\forall t > 0, \ \sigma_0^{t} \succ_{\boldsymbol{w}} \sigma_1^{t} ,
\end{equation}
where $\sigma^{t}$ denotes a trajectory segment starting from time $t$, i.e. $\sigma^{t} = \{s_{t},a_{t},\cdots,s_{t+H-1}, a_{t+H-1}\}$. Here, $\sigma^{t_0}$ and $\sigma^{t}$ are segments with the same state action sequence $(s,a,s',\cdots)$, but starting from the different time.
\end{definition}

\begin{definition}[Transitivity]
Transitivity means that if the teacher prefers $\sigma_0$ over $\sigma_1$ and $\sigma_1$ over $\sigma_2$ under the same weight $\boldsymbol{w}$, then the teacher must also prefer $\sigma_0$ over $\sigma_2$ under weight $\boldsymbol{w}$. Formally, this is expressed as:
\begin{equation}
(\sigma_0 \succ_{\boldsymbol{w}} \sigma_1) \land (\sigma_1 \succ_{\boldsymbol{w}} \sigma_2) \implies \sigma_0 \succ_{\boldsymbol{w}} \sigma_2 .
\end{equation}
This property ensures logical coherence of preferences across multiple trajectory segments. Thus, the teacher's feedback does not contradict itself when extended to multiple comparisons.
\end{definition}

\revise{The symmetry, consistency, and transitivity requirements in Assumption \ref{asp:1} align with standard preference modeling in single-objective RL \cite{pebble}.} 
Then, we assume the presence of a perfect teacher, which can provide the preference over an arbitrary pair of trajectory segments with arbitrary finite length under an arbitrarily given weight. 

\begin{assumption} \label{asp:teacher}
We assume the existence of a teacher who can provide the preference feedback for two arbitrary trajectory segments $(\sigma_0, \sigma_1)$, based on an arbitrary weight vector $\boldsymbol{w}$.    
\end{assumption}

In Assumption \ref{asp:1} and \ref{asp:teacher}, we assume that the teacher can provide preferences $p \in \{0, 0.5, 1\}$ over arbitrary pairs of segments $(\sigma_0, \sigma_1)$ under a given weight $\boldsymbol{w}$, and that these preferences satisfy symmetry, consistency, and transitivity. 
\revise{The assumption of preference availability under given weights is based on existing single-objective preference learning works \cite{pebble, surf}. }
This indicates that the teacher's preferences are based on stable and consistent feedback related to the task objectives.
Based on Assumption \ref{asp:1}, it is reasonable to assume that the task has an underlying true reward, which is aligned with the teacher's preferences. We formalize it in Assumption \ref{asp:2}.
This assumption helps to establish a connection between the teacher's preferences and policy optimization.

\begin{assumption} \label{asp:2}
There exists a true reward function $\boldsymbol{r}$ for a certain multi-objective task, if there exists a teacher that can express preferences for this task. 
Furthermore, the value of the true weighted reward $\boldsymbol{w}^T \boldsymbol{r}$ is bounded by a constant $r_\text{max}$. Formally, this is written as:
\begin{equation}
\max_{\boldsymbol{w},s,a} |\boldsymbol{w}^T \boldsymbol{r}(s,a)| \le r_\text{max}.
\end{equation}
The above equation indicates that regardless of the chosen weight vector $\boldsymbol{w}$, the absolute value of the weighted reward will not exceed this predefined upper limit.
\end{assumption}

\revise{Assumption \ref{asp:2} is a common practice in existing works \cite{melo2001convergence, li2021decentralized, li2024anocbabased}, as most real-world problems involve bounded rewards.} 
By doing this, Assumption \ref{asp:2} prevents issues such as divergence in the reward function, thereby enabling Theorem \ref{thm:1}, as discussed in the following paper.

\begin{assumption} \label{asp:discrete_w}
    The optimal policy $\pi^*(a|s,\boldsymbol{w}_0)$ under weight $\boldsymbol{w}_0$ is also the optimal policy under weight $\boldsymbol{w}\in\{\boldsymbol{w}| \|\boldsymbol{w}-\boldsymbol{w}_0\|_\infty\le \epsilon\}$, $\exists \epsilon>0,\forall s\in\mathcal{S},a\in\mathcal{A},\boldsymbol{w}_0\in\mathcal{W}$.
\end{assumption}

\revise{Assumption \ref{asp:discrete_w} is based on the assumption that the value function is continuous with respect to the weight vector $\boldsymbol{w}$, which is reasonable and commonplace in industrial applications. }
With Assumption \ref{asp:discrete_w}, we could discretize the weight space $\mathcal{W}$ into a finite space $\mathcal{W}^{[N_w]}$ of size $N_w = \frac{|\mathcal{W}|}{\epsilon^m} \le \epsilon^{-m}$, i.e. divide the weight space $\mathcal{W}$ to super cubes with side length $\epsilon$. 
The optimal policies for weights within each super cube are identical. Therefore, $\mathcal{W}^{[N_w]}$ could fully represent $\mathcal{W}$, as they induce the same set of optimal policies.

Under Assumption \ref{asp:1}, \ref{asp:teacher}, \ref{asp:2} and \ref{asp:discrete_w}, in the following theorems, we illustrate that the entire Pareto frontier could be obtained by a simple algorithm (Algorithm \ref{alg:1}) using preferences given different weights.
Specifically, we first prove that any optimal policy in an arbitrary given weight is in the Pareto frontier in Theorem \ref{thm:1}.
Then in Theorem \ref{thm:2}, we prove that the optimal policies in all weights could form any convex Pareto frontier.
Further, for non-convex Pareto frontiers, we prove the frontier could be obtained using preferences collected in designed weights in Theorem \ref{thm:3}.

\begin{theorem} \label{thm:1}
    Each policy in the policy set $\pi^*(a|s,\boldsymbol{w})\in\Pi^*$ derived from Algorithm \ref{alg:1} is in the Pareto frontier when the segment length $H\rightarrow\infty$.
\end{theorem}

\begin{proof}
    We prove this theorem by contradiction. 
    Suppose  $\pi^*(a|s,\boldsymbol{w}) $ is not in the Pareto frontier. Then there must exist a policy  $\pi^{\circ}(a|s,\boldsymbol{w}) \neq \pi^*(a|s,\boldsymbol{w}) $ which dominates  $\pi^*(a|s,\boldsymbol{w}) $.
    And then there must exist a weight  $\boldsymbol{w}_0 $ and a pair of trajectories  $\tau^{\circ} $ and  $\tau^* $ which are generated from  $\pi^{\circ}(a|s,\boldsymbol{w}) $ and  $\pi^*(a|s,\boldsymbol{w}) $ respectively, and  $\tau^{\circ} \succ_{\boldsymbol{w}_0} \tau^* $.
    We extract segments of length  $H $ from  $\tau^{\circ} $ and  $\tau^* $, denoted as  $\sigma^{\circ} $ and  $\sigma^* $ respectively. 
    Under Assumption \ref{asp:1}, the teacher can always output the true preference between two segments.
    
    Let $s_t^{\square}$ and $a_t^{\square}$ denote the state and action at time  $t$ in the trajectory $\tau^{\square}$, where $\square$ is an arbitrary symbol. 
    With discount factor $\gamma $, the difference between the discounted total return  $\sum_{t=0}^\infty \gamma^t \boldsymbol{w}_0^T\boldsymbol{r}(s_t,a_t) $ and the truncated discounted total return $\sum_{t=0}^{H-1} \gamma^t \boldsymbol{w}_0^T\boldsymbol{r}(s_t,a_t) $ is bounded, i.e. 
    $|\sum_{t=H}^\infty \gamma^t \boldsymbol{w}_0^T\boldsymbol{r}(s_t,a_t)| \le \frac{\gamma ^H}{1-\gamma}r_\text{max}$.
    Let $\mathcal{R}_{\underline{t}}^{\bar{t}}(\sigma^\circ)=\sum_{t=\underline{t}}^{\bar{t}} \gamma^t \boldsymbol{w}_0^T\boldsymbol{r}(s_t^\circ,a_t^\circ)$, $\mathcal{R}_{\underline{t}}^{\bar{t}}(\sigma^*)=\sum_{t=\underline{t}}^{\bar{t}} \gamma^t \boldsymbol{w}_0^T\boldsymbol{r}(s_t^*,a_t^*)$, we have
    \begin{align*}
        & \sum_{t=0}^{H-1} \gamma^t \boldsymbol{w}_0^T\boldsymbol{r}(s_t^{\circ},a_t^{\circ}) -\sum_{t=0}^{H-1} \gamma^t \boldsymbol{w}_0^T\boldsymbol{r}(s_t^{*},a_t^{*}) \ge 2\frac{\gamma ^H}{1-\gamma}r_\text{max} \\
        \Leftrightarrow
        & \mathcal{R}_0^{H-1}(\sigma^{\circ})-\mathcal{R}_0^{H-1}(\sigma^{*})\ge 2\frac{\gamma ^H}{1-\gamma}r_\text{max} \\
        & ~~ \ge |\mathcal{R}_H^{\infty}(\sigma^{\circ})|+|\mathcal{R}_H^{\infty}(\sigma^{*})|
        \ge \mathcal{R}_H^{\infty}(\sigma^{*})-\mathcal{R}_H^{\infty}(\sigma^{\circ}) \\
        \Rightarrow
        & \mathcal{R}_0^{\infty}(\sigma^{\circ})-\mathcal{R}_0^{\infty}(\sigma^{*})\ge 0\\
        \Leftrightarrow
        &\sum_{t=0}^\infty \gamma^t \boldsymbol{w}_0^T\boldsymbol{r}(s_t^{\circ},a_t^{\circ})-\sum_{t=0}^\infty \gamma^t \boldsymbol{w}_0^T\boldsymbol{r}(s_t^{*},a_t^{*}) \ge 0.
    \end{align*}
    Therefore, a sufficient condition that the preference between two trajectories is consistent with the preference between the two segments is that $|\sum_{t=0}^{H-1} \gamma^t \boldsymbol{w}_0^T\boldsymbol{r}(s_t^{\circ},a_t^{\circ})-\sum_{t=0}^{H-1} \gamma^t \boldsymbol{w}_0^T\boldsymbol{r}(s_t^{*},a_t^{*})|\ge 2\frac{\gamma ^H}{1-\gamma}r_\text{max}$.
    This condition can always be satisfied when $H \rightarrow \infty$, 
    which means the true preference between two trajectories ( $\tau^{\circ} \succ_{\boldsymbol{w}_0} \tau^* $) can always be obtained from the teacher.
    That contradicts Algorithm \ref{alg:1} which only terminates when  $\nexists \pi^{\circ} $ s.t.  $\pi^{\circ} \succ_{\boldsymbol{w}_0} \pi^* $ and completes the proof.
\end{proof}

In practice, we typically select pairs of segments with distinct behaviors for human comparison, facilitating humans to provide preferences. 
Therefore, it is reasonable to assume that there exists a minimum difference $\delta$ in discounted returns between any two segments, that is, $\exists \delta \ge 0$ such that $|\sum_{t=0}^{H-1} \gamma^t \boldsymbol{w}_0^T \boldsymbol{r}_1(s_t, a_t) - \sum_{t=0}^{H-1} \gamma^t \boldsymbol{w}_0^T \boldsymbol{r}_2(s_t, a_t)| \ge \delta > 0$. 
Under this assumption, we derive the following Corollary \ref{cor:1}.

\begin{corollary} \label{cor:1}
    If all segment pairs are distinct enough, i.e. $\exists \delta\ge0$ s.t. \revise{$|\sum_{t=0}^{H-1} \gamma^t \boldsymbol{w}_0^T\boldsymbol{r}(s_t^1,a_t^1) -\sum_{t=0}^{H-1} \gamma^t \boldsymbol{w}_0^T\boldsymbol{r}(s_t^2,a_t^2)| \ge \delta > 0$ $\forall \sigma_1, \sigma_2$}, 
    then each policy in the policy set $\pi^*(a|s,\boldsymbol{w})\in\Pi^*$ derived from Algorithm \ref{alg:1} is in the Pareto frontier when the segment length $H\ge\log_\gamma\frac{\delta(1-\gamma)}{2r_\text{max}}$.
\end{corollary}

\begin{theorem} \label{thm:2}
    Algorithm \ref{alg:1} obtains the entire convex Pareto frontier, i.e., $\Pi^*$ is the entire convex Pareto frontier.
\end{theorem}
\begin{proof}
    Since the Pareto frontier is convex, for each policy $\pi^*$ on the Pareto frontier, there must exist a weight $\boldsymbol{w}$ s.t. $\boldsymbol{w}^T\bar{\boldsymbol{R}}^*\ge \boldsymbol{w}^T\bar{\boldsymbol{R}}'$, where $\bar{\boldsymbol{R}}^*=\mathbb{E}_\pi\sum_{t=0}^\infty\gamma^t\boldsymbol{r}(s_t,a_t)$ is the expected total discounted return derived by $\pi^*$, and $\bar{\boldsymbol{R}}'$ is that derived by any other policies.
    Using Theorem \ref{thm:1}, the optimal policy under weight $\boldsymbol{w}$ could be obtained by Algorithm \ref{alg:1}.
    Therefore, by traversing $\boldsymbol{w}$, we can traverse each policy on the Pareto frontier.
\end{proof}

\begin{algorithm}[t]
\caption{Using the teacher to obtain non-convex Pareto frontier, based on insertion sort}
\label{alg:2}
\begin{algorithmic}[1]
\FOR{each policy $\pi_i\in\Pi$}
    \FOR{each policy $\pi_j\in\Pi$}
        \IF{for each $\boldsymbol{w}_k\in W_I$, 
        $\boldsymbol{w}_k^T \sum_{(s,a)\sim \sigma_i} 
        \gamma^t\boldsymbol{r}(s_t,a_t) > 
        \boldsymbol{w}_k^T \sum_{(s,a)\sim \sigma_j}
        \gamma^t\boldsymbol{r}(s_t,a_t)$, where $\sigma_i, \sigma_j$ are segments generated by $\pi_i,\pi_j$, }
            \STATE Assign $\pi_i>\pi_j$
        \ENDIF
    \ENDFOR
\ENDFOR
\STATE Use insertion sort, obtain one or multiple biggest policies
\end{algorithmic}
\end{algorithm}

\begin{theorem} \label{thm:3}
    An arbitrary Pareto frontier could be completely obtained
    with preferences under every weight from an identity matrix weight set $W_I=\{\boldsymbol{w}_i~|~ [\boldsymbol{w}_i,\cdots,\boldsymbol{w}_m] = I, i=1,\cdots,m\}$.
\end{theorem}
\begin{proof}
    We prove it by providing a constructive Algorithm \ref{alg:2}.

    If there is only one policy in the policy space $\Pi$, then it is the Pareto frontier.

    If we add a policy $\pi'$ into the current policy space $\Pi$, then $\pi'$ will be compared to all $\pi\in\Pi$, specifically, compared to the current Pareto frontier $\pi\in\Pi^*$ and the non-Pareto frontier $\pi \in \Pi \setminus \Pi^*$. 
    
    If $\pi'$ is not in the Pareto frontier, then $\exists \pi^*\in\Pi^* $ s.t. $\boldsymbol{w}^T\boldsymbol{R}(\sigma')<\boldsymbol{w}^T\boldsymbol{R}(\sigma^*)$ for all $\boldsymbol{w}\in W_I$, where $\boldsymbol{R}(\sigma)=\sum_{t=0,(s_t,a_t)\sim \sigma}^H \gamma^t\boldsymbol{r}(s_t,a_t)$. Thus, through Algorithm \ref{alg:2}, $\pi'$ won't be included in the new Pareto frontier.
    
    If $\pi'$ is in the Pareto frontier, then $\nexists \pi^*\in\Pi^* $ s.t. $ \boldsymbol{w}^T\boldsymbol{R}(\sigma')<\boldsymbol{w}^T\boldsymbol{R}(\sigma^*)$ for all $\boldsymbol{w}\in W_I$. Thus, through Algorithm \ref{alg:2}, $\pi'$ will be included in the new Pareto frontier.
    That completes the proof.
\end{proof}

While linear weighting approaches ($\boldsymbol{w}^T\boldsymbol{R}$) discover only the convex Pareto frontier as in Algorithm \ref{alg:2}, Theorem \ref{thm:3} operates differently. 
By evaluating policies under unit vector weights $W_I$ via pairwise preference comparisons, we directly assess the Pareto dominance relationship. This allows identification of non-convex Pareto-optimal policies.
We provide another proof for Theorem \ref{thm:3} in Appendix \ref{app:proof}.

The theoretical analysis above has demonstrated that the preference-based multi-objective reinforcement learning framework can converge to the Pareto optimal set under specific conditions, providing important guarantees on its performance. Based on these results, we will describe the detailed steps of the algorithm in the next subsection, showing how this framework can be applied to optimize multi-objective policies in practical scenarios.

\subsection{Multi-Objective Reward Modeling}
\label{subsec:MO_reward_model}

Based on the theoretical foundations established in the previous section, we now focus on the practical implementation of Pb-MORL. In particular, we focus on constructing a multi-objective reward model that aligns with human preferences. 
By utilizing the preference data given by the teacher, we can develop an explicit reward model that captures the complexities of human decision-making.

Inspired by the previous work \cite{pbrl_basic2017} in the single-objective scenario, we construct a preference predictor $P_{\psi}[\sigma_0\succ\sigma_1 | \boldsymbol{w}]$, which is designed to predict the preference $p$ given the pair of segments $\sigma_0$ and $\sigma_1$ under the weight $\boldsymbol{w}$, and is parameterized by $\psi$. The preference predictor $P_{\psi}$ can be trained by minimizing the cross-entropy loss:
\begin{equation}
    \label{eq:PbMORL_CEloss}
    \begin{aligned}
        \mathcal{L}^{\mathrm{p}} =
        -\underset{(\sigma_0,\sigma_1,\boldsymbol{w},p) \sim \mathcal{D}}{\mathbb{E}}
        \Big[ & p(0) \log P_\psi[\sigma_0\succ\sigma_1 | \boldsymbol{w}]          \\
              & + p(1) \log P_\psi[\sigma_1\succ\sigma_0 | \boldsymbol{w}] \Big].
    \end{aligned}
\end{equation}
Utilizing the Bradley-Terry model \cite{bradley-terry, pbrl_basic2017}, an explicit reward model $\hat{\boldsymbol{r}}_\psi$ can be constructed to predict the preference as follows:
\begin{equation}
    \label{eq:PbMORL_reward_model}
    P_\psi[\sigma_1\succ\sigma_0 | \boldsymbol{w}] =
    \frac{\exp\sum_t \gamma^t \boldsymbol{w}^T\hat{\boldsymbol{r}}_\psi(s_t^1,a_t^1)}
    {\sum_{i\in\{0,1\}}\exp\sum_t \gamma^t \boldsymbol{w}^T\hat{\boldsymbol{r}}_\psi(s_t^i,a_t^i)}.
\end{equation}
\revise{Eq. \eqref{eq:PbMORL_reward_model} models preferences as probabilistic outcomes, thereby accommodating the inherent ambiguity found in human judgments.}
Specifically, it suggests that the preference is exponentially related to the reward sum over the segment. 
Then, the reward model $\hat{\boldsymbol{r}}_\psi$ is trained to predict the preference under the weight $\boldsymbol{w}$.
Although the estimator $\hat r_\psi$ is not inherently a binary classifier, the process of learning this estimator can be regarded as a binary classification, where the preferences $p$ serve as the classification labels.

In the previous discussion, we introduce how to leverage the preference data to construct a reward model. 
Theoretically, when the reward model $\boldsymbol{r}$ aligns perfectly with the teacher's preferences, we can directly optimize this model to derive the optimal policy. 
To formalize this relationship, we present the following theorem:

\begin{theorem} \label{thm:4}
If the reward model $\hat{\boldsymbol{r}}$ is perfectly aligned with the teacher's preferences, that is, for segments $(\sigma_0, \sigma_1)$ with arbitrary length $H$, 
\begin{equation}
\begin{aligned}
& \sigma_0 \succ_{\boldsymbol{w}} \sigma_1 
\iff \\
& ~~~ \sum_{(s,a)\sim \sigma_0} \gamma^t \boldsymbol{w}^T \hat{\boldsymbol{r}}(s_t,a_t) >
\sum_{(s,a)\sim \sigma_1} \gamma^t \boldsymbol{w}^T \hat{\boldsymbol{r}}(s_t,a_t).
\end{aligned}
\end{equation}
Since the segment length $H$ can be arbitrarily long, the above equation is equivalent to
\begin{equation}
\begin{aligned}
& \pi_0 \succ_{\boldsymbol{w}} \pi_1 
\iff \\
& ~~~ \mathbb{E}_{\tau\sim \pi_0} \sum_{t=0}^\infty \gamma^t \boldsymbol{w}^T \hat{\boldsymbol{r}}(s_t,a_t) >
\mathbb{E}_{\tau\sim \pi_1} \sum_{t=0}^\infty \gamma^t \boldsymbol{w}^T \hat{\boldsymbol{r}}(s_t,a_t).
\end{aligned}
\end{equation}
Then, under a given weight vector $\boldsymbol{w}$, maximizing the discounted return
\begin{equation}
J(\pi) = \sum_{t=0}^{\infty} \gamma^t \boldsymbol{w}^T \hat{\boldsymbol{r}}(s_t, a_t)
\end{equation}
is equivalent to selecting the optimal policy \revise{$\pi^*(\cdot | \cdot, \boldsymbol{w})$}.

\begin{proof}
For contradiction, assume that there exists another policy $\pi'$ that performs better than $\pi^*$ under the weight vector $\boldsymbol{w}$, i.e.,
\begin{equation}
\sum_{t=0}^{\infty} \gamma^t \boldsymbol{w}^T \hat{\boldsymbol{r}} (s_t', a_t')
> 
\sum_{t=0}^{\infty} \gamma^t \boldsymbol{w}^T \hat{\boldsymbol{r}} (s_t^*, a_t^*),
\end{equation}
where $(s_t', a_t')$ and $(s_t^*, a_t^*)$ are from the trajectories generated by policies $\pi'$ and $\pi^*$, respectively. In this case, the teacher would prefer the trajectory of $\pi'$ over that of $\pi^*$.

However, since the reward model $\boldsymbol{r}$ is perfectly aligned with the teacher, we have:
\begin{equation}
\begin{aligned}
& \pi' \prec_{\boldsymbol{w}} \pi^*  \Longrightarrow
\\
& ~~~ 
\sum_{t=0}^{\infty} \gamma^t \boldsymbol{w}^T \hat{\boldsymbol{r}} (s_t', a_t') < 
\sum_{t=0}^{\infty} \gamma^t \boldsymbol{w}^T \hat{\boldsymbol{r}} (s_t^*, a_t^*),
\end{aligned}
\end{equation}
which contradicts the fact that $\pi^*$ is the optimal policy under the reward model $\hat{\boldsymbol{r}}$. Therefore, the assumption is false, and the theorem holds.
\end{proof}
\end{theorem}

\subsection{MORL based on Multi-Objective Reward Model}

Having established the construction method of the reward model $\hat{\boldsymbol{r}}_\psi$, we now focus on implementing the Pb-MORL algorithm. Specifically, we leverage the learned reward model as a substitute for the traditional reward function, enabling the direct application of existing MORL techniques.

In typical MORL training process, the algorithm collects transitions $(s, a, s', \boldsymbol{r}, \boldsymbol{w})$, which are composed of state $s$, action $a$, next state $s'$, multi-objective reward $\boldsymbol{r}$ and the weight vector $\boldsymbol{w}$. These transitions are then utilized to update value functions and policies.
In contrast, our method collects transitions where the reward $\boldsymbol{r}$ is replaced with the predicted reward from the model, $\boldsymbol{\hat{r}}_\psi$. This allows us to align the policy with preference data by minimizing the loss as Eq. \eqref{eq:PbMORL_CEloss}. 
\revise{This leads to a straightforward implementation of Pb-MORL. 
We first train the multi-objective reward model, followed by conducting MORL training based on this reward model. }

However, directly using the reward model to train an MORL agent may potentially result in inefficient policy learning:
\begin{enumerate}
    \item Insufficient amount of preference data: To train a high-quality reward model, a substantial amount of preference data may need to be collected beforehand.
    \item \revise{Imprecise reward model: 
    When the amount of preference data is insufficient, the reward model may overfit the limited training data, resulting in an imprecise reward model and consequently leading to suboptimal policy performance.}
\end{enumerate}

Below are two techniques that can help improve sample efficiency and performance.
\begin{itemize}
    \item Continuous preference collection: Continuously gather preference data during training, which can enrich the training data of the reward model.
    \item Relabeling: Relabel historical data with the updated reward model, which can increase the sample efficiency of preference and transition data.
\end{itemize}

Based on the above techniques, we present Algorithm \ref{alg:Pb-EQL}, which is a variant of the Pb-MORL approach discussed above. 
By integrating the Envelope multi-objective Q-learning (EQL) \cite{EQL} into our learning process, Algorithm \ref{alg:Pb-EQL} achieves a simple yet effective approach for policy optimization.
Specifically, in lines 3-13, the agent interacts with the environment to collect transition data.
In lines 14-22, the multi-objective reward model is updated continuously during policy training.
In line 23, the rewards of the transition data in the replay buffer are relabeled.
In lines 26-29, the Q-function and policy are updated using the EQL method.

\begin{algorithm}[htbp]
\caption{Pb-MORL algorithm using the EQL method}
\label{alg:Pb-EQL}
\begin{algorithmic}[1]
\REQUIRE Frequency of teacher feedback $K$, number of sampled segment $N_s$, number of sample weights $N_w$ for reward learning, timesteps for learning start $T_0$
\ENSURE Multi-objective reward model $\boldsymbol{\hat{r}}_\psi$, multi-objective Q function $\boldsymbol{Q}_{\theta}$, policy \revise{$\pi_\phi(\cdot | \cdot, \boldsymbol{w})$}
    \STATE Initialize parameter vectors $\psi, \theta, \phi$
    \FOR{each iteration}
        \FOR{each environment step $t$}
            \STATE Obtain current state $s_t$
            \IF{global step $<T_0$}
                \STATE Randomly sample action $a_t$
            \ELSE
                \STATE Obtain action \revise{$a_t\sim \pi_\phi(\cdot | s_t, \boldsymbol{w})$} under a weight $\boldsymbol{w}$
            \ENDIF
            \STATE Obtain transition $(s_t,a_t,s_{t+1}, \boldsymbol{w})$ 
            \STATE Obtain multi-objective reward $\hat{\boldsymbol{r}}(s_t,a_t)$
            \STATE Add $(s_t,a_t,s_{t+1},\hat{\boldsymbol{r}}(s_t,a_t), \boldsymbol{w})$ into replay buffer $\mathcal{D}$
        \ENDFOR
        \IF{iteration mod K == 0}
            \STATE Sample $N_s$ query $(\sigma_0, \sigma_1)\sim \mathcal{D}$
            \STATE Sample $N_w$ weights $\boldsymbol{w}$
            \STATE Query overseer for preference $p$ for all queries $(\sigma_0, \sigma_1)$ under all $\boldsymbol{w}$
            \STATE Store all $N_s\times N_w ~ (\sigma_0, \sigma_1, \boldsymbol{w}, p)$ to buffer $\mathcal{D}_p$
            \FOR{each gradient step}
                \STATE Sample minibatch $(\sigma_0, \sigma_1, \boldsymbol{w}, p)\sim \mathcal{D}_p$
                \STATE Optimize Eq. \eqref{eq:PbMORL_CEloss} to update reward model $\boldsymbol{\hat{r}}_\psi$
            \ENDFOR
            \STATE Relabel entire replay buffer $\mathcal{D}$ using $\boldsymbol{\hat{r}}_\psi$
        \ENDIF
        \FOR{each gradient step}
            \STATE Sample a minibatch from replay buffer $\mathcal{D}$
            \STATE Update Q function $\boldsymbol{Q}_{\theta}$ by minimizing $|\boldsymbol{Q}-B\boldsymbol{Q}|$ under $(s,a,s',\boldsymbol{r}, \boldsymbol{w}) \sim \mathcal{D}$, as Eq. \eqref{eq:B_optimal}
            
            \STATE \revise{Update the Q-learning policy \revise{$\pi_\phi(\cdot | \cdot, \boldsymbol{w})$} by maximizing $\boldsymbol{w}^T \boldsymbol{Q} (s, \pi_\phi(\cdot | s, \boldsymbol{w}), \boldsymbol{w})$ under $(s,a,s',\boldsymbol{r}, \boldsymbol{w}) \sim \mathcal{D}$}
        \ENDFOR
    \ENDFOR
\end{algorithmic}
\end{algorithm}

\begin{figure*} [t]
    \centering
    \includegraphics[width=0.85\textwidth]{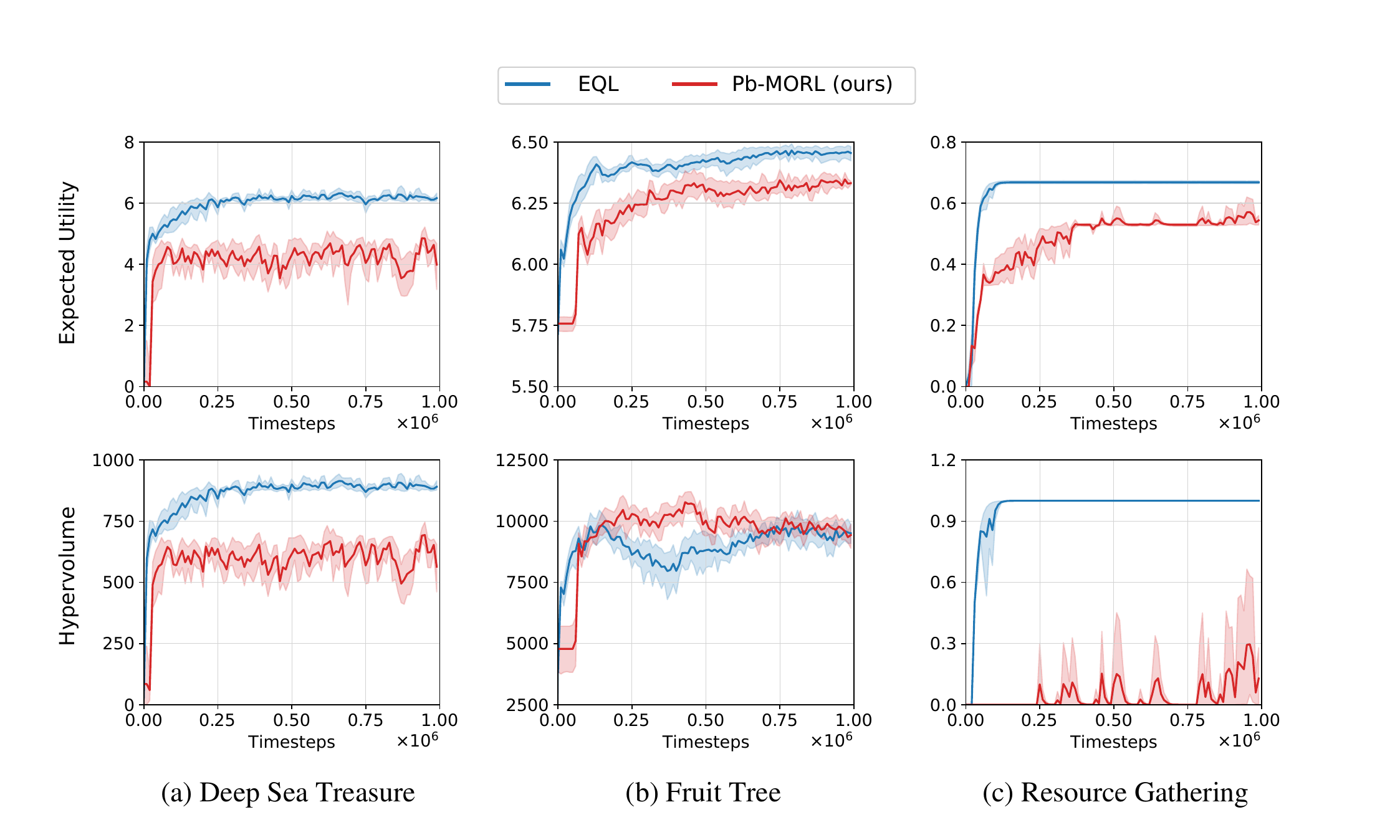}
    \caption{The training curves of the expected utility and hypervolume on three multi-objective benchmark tasks. The experiments are conducted on 5 random seeds. Blue: the oracle method (EQL). Red: our method. }
    \label{fig:curve_mogym} 
\end{figure*}

\section{Experimental Results} \label{sec:experiment}

\subsection{Setups}
\label{subsec:setup}

In this section, we conduct several experiments to evaluate the effectiveness of the proposed method. 
We test Pb-MORL on several benchmark multi-objective tasks \cite{EQL, PGMORL, muni_cac2024} to demonstrate its effectiveness across diverse multi-objective settings. \revise{Additionally, we evaluate Pb-MORL on a custom task for multi-energy storage system charging and discharging as well as an autonomous driving task on a multi-line highway}, showing its potential for real-world industrial applications.

\textbf{Construction of the multi-objective teacher.}
Similar to prior PbRL works \cite{surf, pebble, clarify}, in order to systemically evaluate the performance, we construct a ``scripted teacher'', which provides preferences $p$ between two trajectory segments $\sigma_0,\sigma_1$ under certain weight $\boldsymbol{w}$ according to the task's ground truth reward function. 
The following paragraph formalizes the process used by the scripted teacher in multi-objective RL.

Let $\boldsymbol{r}_{gt}$ denote the task's ground truth reward function.
Then, for the segment pair $(\sigma_0,\sigma_1)$, the scripted teacher first computes the discounted reward sum:
\begin{equation}
\boldsymbol{R}_i = \sum_{t=0}^{H-1} \gamma^t \boldsymbol{r}_{gt}(s^i_{t},a^i_{t}) \quad i=0,1, 
\end{equation}
where $\gamma$ is the discount factor, $t$ is the time step, $s^i_{t},a^i_{t}$ represents the state and action of segment $\sigma_i$ in time step $t$, and $H$ is the segment length.
Next, the teacher computes the weighted inner product with the given weight vector $\boldsymbol{w}$:
\begin{equation}
    R_i = \boldsymbol{w}^T \boldsymbol{R}_i = \boldsymbol{w}^T \left( \sum_{t=0}^{H-1} \gamma^t \boldsymbol{r}_{gt}(s^i_{t},a^i_{t}) \right)\quad i=0,1,
\end{equation}
The scripted teacher compares $R_0$ and $R_1$ to determine which segment performs better:
\begin{equation}
p =
\begin{cases}
1, & \text{if } R_0 > R_1, \\
0.5, & \text{if } R_0 = R_1, \\
0, & \text{if } R_0 < R_1.
\end{cases}
\end{equation}
Since the scripted teacher’s preferences directly correspond to the task's ground truth reward, the algorithms can be quantitatively evaluated using the ground truth reward function.

\textbf{Evaluation metrics.} We use two metrics to evaluate the empirical performance on each task:

\begin{enumerate}
    \item \textbf{Expected Utility (EU)} \cite{MO-gym}: This metric measures the average utility under randomly sampled weights. Let $\boldsymbol{w}$ be a weight vector randomly sampled from the uniform distribution in $\mathcal{W}$ space. Let $U(\pi, \boldsymbol{w})$ represent the utility function of policy \revise{$\pi(\cdot|\cdot, \boldsymbol{w})$} under the weight $\boldsymbol{w}$, which is usually the inner product of the discounted return and the weight $\boldsymbol{w}$. The expected utility $\mathrm{EU}(\pi)$ is then defined as:
    \begin{equation}
    \mathrm{EU}(\pi) = \mathbb{E}_{\boldsymbol{w}} U(\pi, \boldsymbol{w}).
    \end{equation}
    Expected Utility is crucial for evaluation, as it comprehensively measures a policy's overall performance across objectives. 
    Unlike hypervolume \cite{hypervolume}, which focuses on boundary solutions, EU evaluates the policy's average behavior over the entire weight space. Thus, it serves as a more relevant indicator of general performance in many multi-objective tasks.

    \item \textbf{Hypervolume (HV)} \cite{hypervolume} : Given an approximate Pareto Frontier set $\tilde{\mathbf{F}}$ of multi-objective return and a reference point $\boldsymbol{R}_{\mathrm{ref}}$, the hypervolume metric is defined as:
    \begin{equation}
    \mathrm{HV}(\tilde{\mathbf{F}},\boldsymbol{R}_{\mathrm{ref}}) = 
    \bigcup_{\boldsymbol{R} \in \tilde{\mathbf{F}}} \mathrm{volume}(\boldsymbol{R}_{\mathrm{ref}}, \boldsymbol{R}),
    \end{equation}
    where $\mathrm{volume}(\boldsymbol{R}_{\mathrm{ref}}, \boldsymbol{R})$ is the volume of the hypercube spanned by the reference vector $\boldsymbol{R}_{\mathrm{ref}}$ and the vector $\boldsymbol{R}$. The reference point here is typically an estimation of the worst possible return for all objectives.
\end{enumerate}

\textbf{Experimental details.} 
For implementation details, for line 12 of Algorithm \ref{alg:Pb-EQL}, we use a query-policy aligned replay buffer to maintain an accurate reward model in the near-policy region \cite{QPA}.
For line 17, we use the scripted teacher mentioned earlier to generate preference data.
The detailed hyperparameter settings of Pb-MORL are shown in Table \ref{tab:hyperparam}.
For the baseline, we use EQL \cite{EQL} as an oracle method, which leverages the ground truth reward function for policy learning.

\begin{table}[t]
\caption{Hyperparameter settings for Pb-MORL}
\label{tab:hyperparam}
\centering
\begin{tabular}{lc}
\toprule
\multicolumn{1}{c}{Hyperparameter} & Value\\
\midrule
Preference frequency $K$            & $500$ \\
Number of sampled segment $N_s$     & $300$ \\
Number of sample weights $N_w$      & $10$ \\
Discount factor $\gamma$            & $0.99$ \\
Batch size                          & $256$ \\
Learning rate                       & $3\times 10^{-4}$ \\
Training timesteps                  & $1\times 10^6$ \\
Number of Q network hidden layers   & $2$ \\
Number of hidden units per layer    & $128$ \\
Q target update $\tau$              & $1\times 10^{-4}$ \\
Optimizer                           & Adam \\
\bottomrule
\end{tabular}
\end{table}

\begin{figure} [t]
    \centering
    \includegraphics[width=0.6\linewidth]{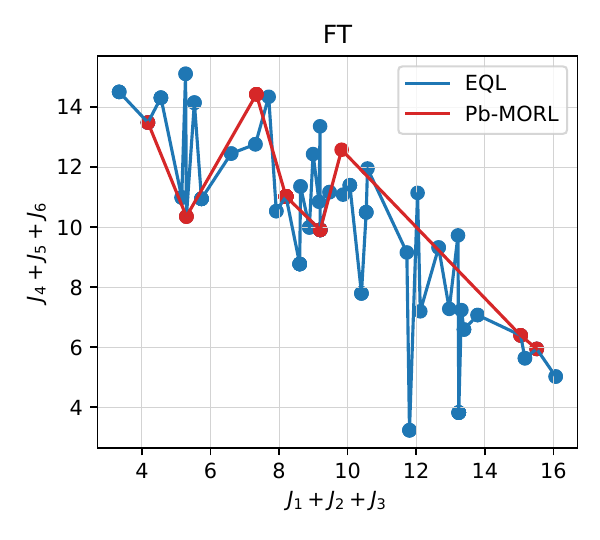}
    \caption{Visualization of the estimated Pareto frontier of two methods in FT. Note that the actual Pareto frontier in FT has 6 dimensions, we add up the first 3 and last 3 dimensions of rewards for illustration.}
    \label{fig:pareto} 
\end{figure}

\subsection{Experimental Results on Multi-Objective Benchmark Tasks}
\label{subsec:results_mogym}

\textbf{Tasks. }  
We evaluate our method on three multi-objective benchmark tasks \cite{MO-gym}, each presenting distinct challenges, such as balancing time-cost and total reward or optimizing independently across multiple objectives:
\begin{itemize}
    \item \textbf{Deep Sea Treasure (DST) \cite{deep_see_treasure}}: 
    An agent controls a submarine in a 10$\times$11 grid to discover treasures, balancing time and treasure value. The grid contains 10 treasures, with the value increasing with the distance from the starting point $s_0 = (0, 0)$. The multi-objective reward $\boldsymbol{r}(s, a)$ has two dimensions: $r_1(s, a)$ for treasure value and $r_2(s, a)$ for time cost, decreasing by 1 for each step. 
   
    \item \textbf{Fruit Tree (FT) \cite{EQL}}: 
    A full binary tree provides a six-dimensional reward $\boldsymbol{r}\in \mathbb{R}^6$ at each leaf, representing nutritional components: \textsc{\{Protein, Carbs, Fats, Vitamins, Minerals, Water\}}. The agent maximizes utility for a given weight by selecting the optimal path from root to leaf while choosing between the left and right subtrees. 
    
    \item \textbf{Resource Gathering (RG) \cite{resource_gathering}}: 
    An agent collects the gold or gem in a 5$\times$5 grid while evading two enemies. Encountering an enemy in the same cell poses a 10\% risk of death. The multi-objective reward $\boldsymbol{r}(s, a)$ has three dimensions: $r_1(s, a) = -1$ if being killed, $r_2(s, a) = +1$ if safely returning home with gold, $r_3(s, a) = +1$ if returning with gem. 
\end{itemize}

To justify the selection of $H$ values for each task, we analyze their characteristics. 
For the DST task, where episode length varies and rewards accumulate over time, we choose $H = 7$ to capture the cumulative effects. 
In the FT task, with a fixed episode length of $6$, $H = 6$ is suitable to encompass the full trajectory. 
For the RG task, which features sparse rewards and early episode termination, we select $H = 10$ to ensure enough steps are available to differentiate between policies.

Figure \ref{fig:curve_mogym} presents the expected utility and hypervolume results for the three tasks. 
In the DST task, our method performs comparably to the oracle in expected utility, demonstrating consistent utility improvement over time. 
For the FT task, our method matches the oracle in expected utility and surpasses it in hypervolume, indicating effective use of preference for enhancing the Pareto frontier quality. 
In the RG task, while our method's utility approaches optimal performance, the hypervolume results are less favorable.
This may be because the returns of RG are limited to $0$ or $\pm1$, and the learned reward model exhibits imprecision in capturing edge-cases under these sparse rewards, which restricts hypervolume growth.

To demonstrate the high quality of the Pareto frontier we learned, we visualized the Pareto frontier learned by EQL and our method in the FT task, as shown in Fig. \ref{fig:pareto}. 
Our method captures the key factors of the Pareto frontier of the oracle method, showing its effectiveness for application in practice.

\begin{figure}[htbp]
    \centering
    \includegraphics[width=\linewidth]{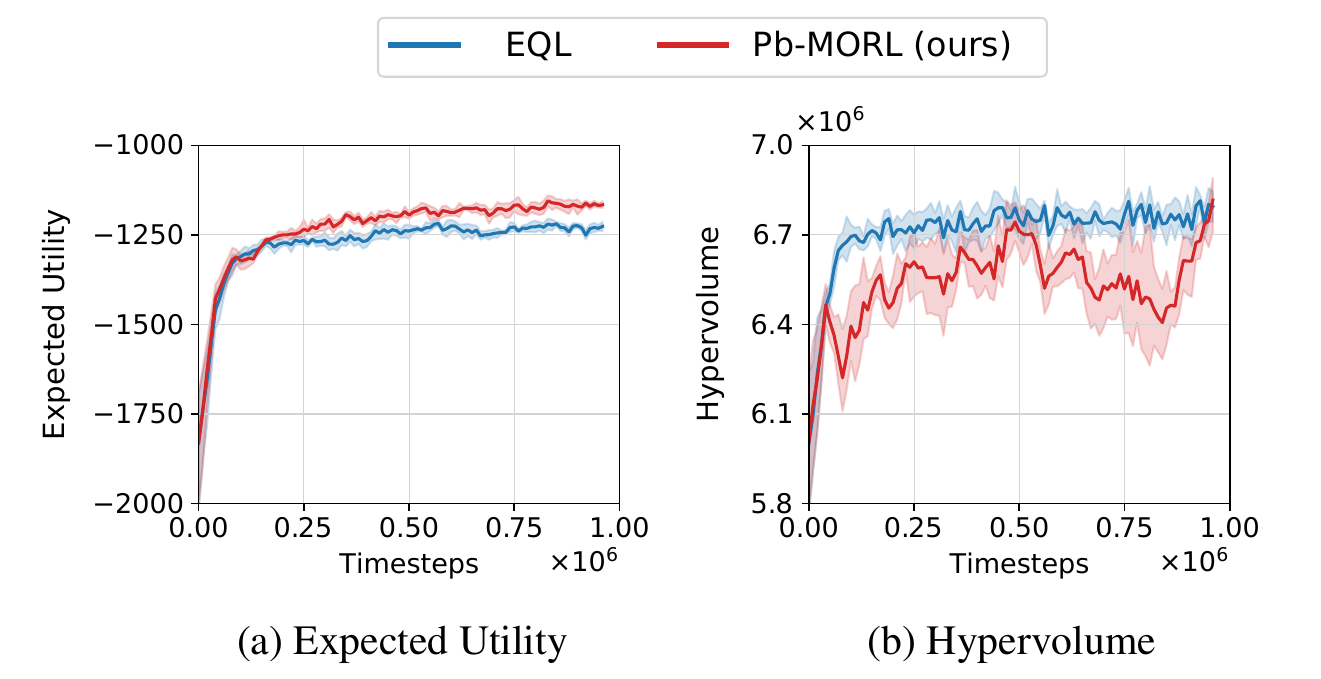}
    \caption{The training curves of the expected utility and hypervolume on the energy system task, averaging over 5 random seeds. Blue: the oracle method (EQL). Red: our method. }
    \label{fig:curve_energy} 
\end{figure}

\subsection{Experimental Results on the Custom Energy Task}
\label{subsec:results_energy}

\textbf{The multi-energy management task. }  
To assess the potential of Pb-MORL for real-world industry applications, we designed a custom multi-objective task for multi-energy storage, simulating the charging and discharging of an energy storage system. The agent controls discharge and charge levels to satisfy external energy demands while balancing cost savings and system lifespan.

\begin{itemize}
\item \textbf{State space}: 
The state space includes four scalar values:
Current stored energy $s_\text{storage}$ (kWh), 
current energy generated from renewable sources $s_\text{new}$ (kWh), 
external energy demand $s_\text{demand}$ (kWh), 
and the electricity market price $s_\text{price}$ (monetary units).
Thus, the state vector can be represented as: 
\begin{equation}
s = [s_\text{storage}, s_\text{new}, s_\text{demand}, s_\text{price}].
\end{equation}

\item \textbf{Action space}:
The action $a$ is a scalar indicating the discharge level. Positive values represent energy discharged to meet external demand, while negative values indicate energy charged from renewable sources or the grid.

\item \textbf{Transition}: 
After a state transition, the new storage level is calculated as:
\begin{equation}
s_{\text{storage}, t+1} = \min(s_{\text{storage}}^\text{max}, (s_{\text{storage}, t} - {a}_t)^+),
\end{equation}
where $s_{\text{storage}}^\text{max}$ is the maximum capacity of the energy storage, and $(\cdot)^+$ denote $\max(\cdot,0)$.

\item \textbf{Reward function}: 
The reward is a two-dimensional vector, where 
the first dimension $r_1(s_t, a_t)$ penalizes the electricity purchasing cost.
At each time step, the system may purchase energy to satisfy the external energy demand and charge the storage.
The amount of energy bought for charging is 
\begin{equation}
    g_\text{charge}=
    \begin{cases}
        (-a-(s_\text{new}-s_\text{demand})^+)^+& a<0\\
        (a-s_\text{storage})^+& a\ge0\\
    \end{cases},
\end{equation}
and that for external demand is 
\begin{equation}
    g_\text{demand}=
    ((s_\text{demand}-s_\text{new})^+-(a)^+)^+,
\end{equation}
$r_1(s_t, a_t)$ is calculated as $r_1(s_t, a_t)=s_{\text{price}}\times(g_\text{demand}+g_\text{charge})$.
The second dimension $r_2(s_t, a_t)$ indicates a penalty for discharging: $r_2(s_t, a_t)=-1$ when energy is discharged and $0$ otherwise. This design aims to reduce discharges, thus prolonging the system's lifespan.

\end{itemize}

In this task, rewards are cumulative, with a maximum episode length of $50$. Since the agent does not face failures leading to early termination, a sufficiently long $H$ is critical for capturing long-term policy performance. Setting $H = 10$ allows for comprehensive observation of cumulative returns, enabling effective differentiation among policy performances during optimization.

Figure \ref{fig:curve_energy} presents the experimental results of the multi-energy management task. 
Our method surpasses the oracle method in expected utility. This can be attributed to the task's inherent randomness and complex transition dynamics, which make it challenging to directly optimize task rewards like electricity costs or the lifespan loss of system charging. 
Instead, preference provides a more flexible way to guide policy optimization, allowing the policy to adapt to system complexities more effectively. 
Additionally, our method matches the oracle in the hypervolume metric.

The weight-conditioned policy $\pi(a|s, w)$ learned by Pb-MORL allows dynamic adaptation to changing user preferences.
For instance, in energy management, operators may prioritize cost reduction during peak pricing periods ($w_1 \uparrow$) and system longevity during high-stress operations ($w_2 \uparrow$). 
Since Pb-MORL trains a single policy conditioned on arbitrary weights $w \in \mathcal{W}$, adapting to such changes only requires modifying the input weight vector $w$ at deployment, eliminating the need for policy retraining. 
Similarly, in autonomous driving in Section \ref{subsec:results_highway}, safety weights can be adjusted during adverse weather by simply updating $w$. 
This adaptability minimizes operational overhead and enables Pb-MORL to respond instantly to evolving objectives, making it well-suited for dynamic environments with non-stationary preferences.

\begin{figure}[htbp]
    \centering
    \includegraphics[width=0.9\linewidth]{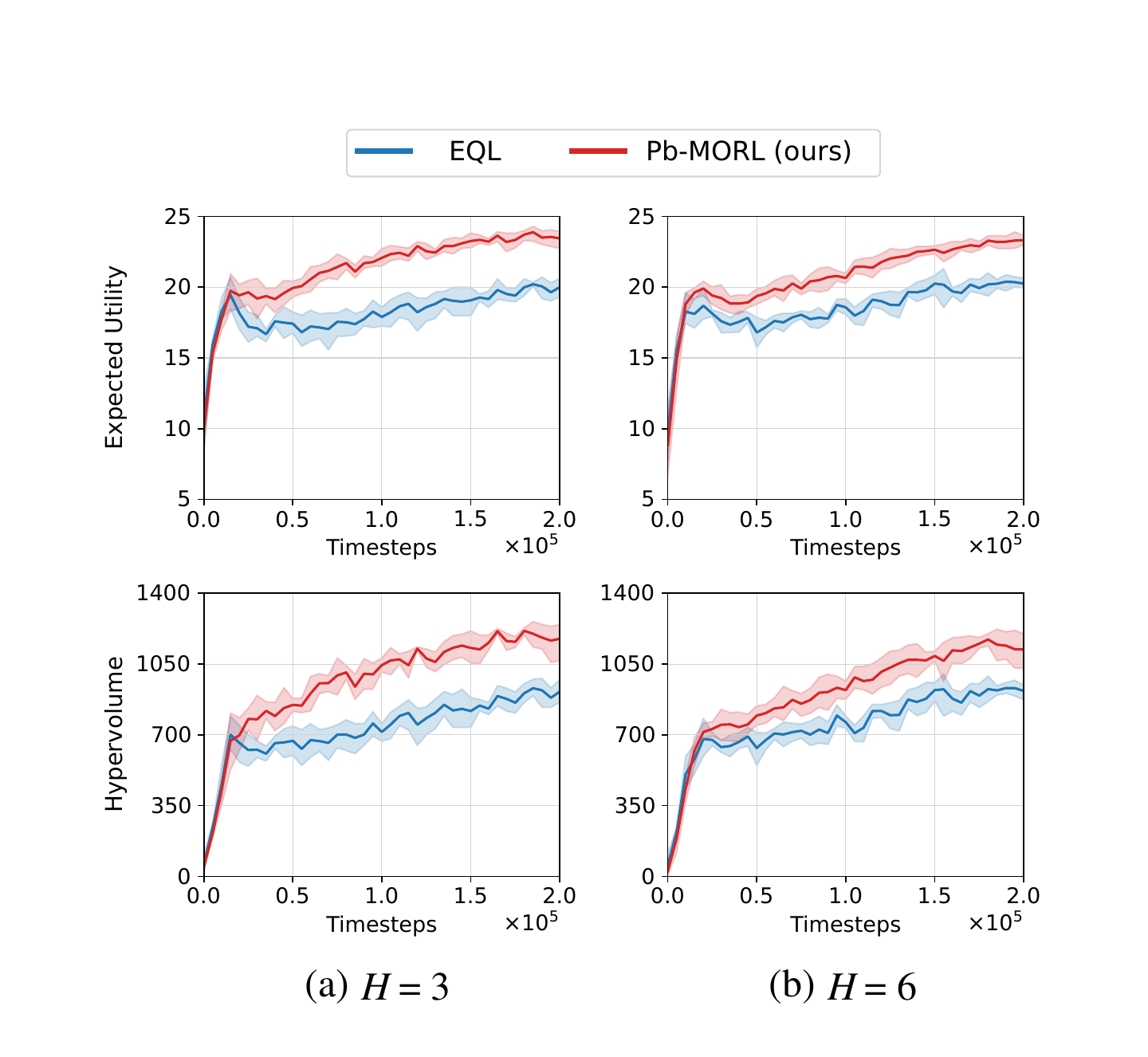}
    \caption{\revise{The training curves of the expected utility and hypervolume on the highway task, averaging over 5 random seeds. Blue: the oracle method (EQL). Red: our method.} }
    \label{fig:curve_highway} 
\end{figure}

\subsection{\revise{Experimental Results on the Multi-Lane Highway Task}}
\label{subsec:results_highway}

\revise{
\textbf{The multi-lane highway task. }  
To validate the effectiveness of our approach in real-world complex control scenarios, we evaluate it in a multi-lane highway task \cite{MO-gym, highway}. In this task, the agent navigates a three-lane highway while driving as quickly as possible, avoiding collisions and prioritizing positioning in the rightmost lane. 
This setting comprehensively tests the agent's ability to perform in dynamic and multi-faceted environments.
}
\begin{itemize}
\item \revise{\textbf{State space}: 
The state is represented by a $V \times 5$ matrix that includes the coordinates and speeds of the ego vehicle and $V-1$ surrounding vehicles. Each line consists of [presence of the vehicle, $x$ coordinate, $y$ coordinate, $x$ velocity, $y$ velocity].
}

\item \revise{\textbf{Action space}:
Actions are categorized discretely as follows: lane change to the left (0), maintaining the current state (1), lane change to the right (2), acceleration (3), and deceleration (4). 
These actions are integrated with a lower-level controller for speed and steering.
}

\item \revise{\textbf{Transition}:
The vehicle kinematic is modeled using a simplified kinematic bicycle model, which assumes the left and right wheels function as a single wheel. It regards the front wheel as the primary steering control while omitting sliding effects, thus enabling a more straightforward representation of vehicle dynamics. 
This model formulation captures the essential dynamics of real-world vehicle behavior, enhancing the simulation's fidelity.
The following equations describe the vehicle's motion:
\begin{align*}
\begin{cases}
\dot{x} &= v \cos(\theta + \beta), \\
\dot{y} &= v \sin(\theta + \beta), \\
\dot{v} &= a, \\
\dot{\theta} &= \frac{v}{l_r} \tan(\delta), \\
\beta &= \arctan\left(\frac{l_r}{l_f + l_r} \tan(\delta)\right).
\end{cases}
\end{align*}
The surrounding vehicles are controlled by the Intelligent Driver Model (IDM) \cite{idm} and the Minimum Overall Braking Distance (MOBIL) model \cite{mobil}. 
}

\item \revise{\textbf{Reward function}: 
The reward function comprises a three-dimensional vector. 
The first element represents speed reward, calculated as $\frac{v - v_{\text{min}}}{v_{\text{max}} - v_{\text{min}}}$, where $v$ is the current speed of the ego vehicle, and $v_{\text{min}}$ and $v_{\text{max}}$ denote the minimum and maximum allowable speeds, respectively. 
The second element indicates lane position reward, as $1$ if the ego vehicle is in the rightmost lane and $0$ otherwise. 
The third element is a collision penalty, assigned $-1$ upon collision and $0$ otherwise. 
}

\end{itemize}

\revise{
We selected $H = 3$ and $H = 6$ for our evaluation. $H = 3$ corresponds to a low-level behavior over a 5-second period, while $H = 6$ represents longer driving behavior, allowing for a more comprehensive assessment of the agent's performance. By evaluating our method with both time horizons, we obtain a more nuanced understanding of its capabilities across different driving scenarios.
Following prior works \cite{highway2e5_1, highway2e5_2}, we trained the agent for 200,000 steps. 
}

\revise{
As shown in Figure \ref{fig:curve_highway}, our method surpasses the oracle method regarding both expected utility and hypervolume. 
In contrast to EQL, which experiences a significant performance decline after an initial increase followed by a slow recovery, our approach maintains stable performance without such setbacks. 
The initial drop in EQL may be because the agent becomes overly focused on immediate goals, such as speed, resulting in aggressive policies that often neglect safety. Consequently, when the negative impact of collision occurs, the agent must re-explore to find more stable policies.
In contrast, Pb-MORL addresses this overfitting through preference-driven reward learning. This method emphasizes the relative benefits of multiple objectives (e.g., ``safe overtaking $\succ$ aggressive overtaking'') rather than focusing on absolute value differences, enabling the reward model to learn the trade-offs of specific scenarios. 
Additionally, continuous preference feedback helps to recalibrate the reward model in three phases: balancing objectives, refining scene-specific policies and optimizing for long-tail risks. This approach effectively prevents oscillations caused by conflicting targets.
}

\revise{
In summary, Sections \ref{subsec:results_energy} and \ref{subsec:results_highway} demonstrate that our preference-guided policy outperforms the oracle across multiple metrics and adapts more effectively to complex systems than direct task reward optimization. 
Additionally, the resulting multi-objective policies exhibit strong interpretability, clearly showing how weight vectors impact policy behavior. 
These findings underscore the potential of our approach for optimizing complex real-world systems.
}

\section{Conclusion} \label{sec:conclusion}

This paper presents the preference-based multi-objective reinforcement learning (Pb-MORL) algorithm, which leverages preference data to overcome the limitations of complicated reward design.
Our contributions include a theoretical proof of optimality, showing the Pb-MORL framework can guide the learning of Pareto-optimal policies. 
In addition, we construct an explicit multi-objective reward model that directly aligns with user preferences, enabling more intuitive decision-making in complex scenarios. 
Extensive experiments demonstrate the effectiveness and interpretability of Pb-MORL in optimizing various types of multi-objective tasks. 
Through this work, we highlight the potential of preference-based frameworks in enhancing multi-objective optimization.

Future research can explore several directions. 
\revise{First, we recognize that some assumptions in our work may not hold in practical scenarios. Specifically, for the symmetry, consistency, and transitivity requirements in Assumption \ref{asp:1}, we can explore non-transitive cases through pairwise ranking methods \cite{lire} and utilize preference aggregation strategies to address violations of other properties.
Second, Assumption \ref{asp:teacher} could be relaxed through active query strategies \cite{mu2024sepoa} that optimize comparison requests.}
Third, to expand its utility in complex systems, we aim to apply Pb-MORL to various domains, such as financial investment and smart manufacturing. 
Notably, we provide an alternative perspective in Appendix \ref{app:colored_glasses}, discussing the motivation for Pb-MORL, highlighting the impact of human subjectivity in preference data on learning quality in traditional PbRL.

\section*{Acknowledgments}
This work is supported by the Beijing Natural Science Foundation (L233005), NSFC (No. 62125304, 62192751), the National Key Research and Development Program of China (2022YFA1004600), the 111 International Collaboration Project (B25027), and the BNRist project (BNR2024TD03003).

{
\appendices

\section{Additional Proof} \label{app:proof}

\begin{proof}[Another proof of Thm.\ref{thm:3}]
    We prove it by providing a constructive Algorithm \ref{alg:21}.
    
    If policy $\pi_i$ is never added into $\Pi^*$, there must exists a policy $\pi_a~a<i$ and $\boldsymbol{w}_k\in W_I$ s.t. $\boldsymbol{w}_k^T \boldsymbol{r}(\sigma_a)>\boldsymbol{w}_k^T \boldsymbol{r}(\sigma_j)$. Therefore $\pi_i$ is dominated by $\pi_a$, thus not being in the Pareto frontier.
    
    If policy $\pi_i$ is added into $\Pi^*$, and removed when traversing $\pi_b$, there must exists a $\boldsymbol{w}_k\in W_I$ s.t. $\boldsymbol{w}_k^T \boldsymbol{r}(\sigma_b)>\boldsymbol{w}_k^T \boldsymbol{r}(\sigma_j)$. Therefore $\pi_i$ is dominated by $\pi_b$, thus not being in the Pareto frontier.
    
    If policy $\pi_i$ is added into $\Pi^*$, and remains in the $\Pi^*$. Assume there is $\pi_c$ dominates $\pi_i$. If $c>i$, when traversing $\pi_c$, $\pi_i$ must be in $\Pi^*$ and be removed, which contradicts the algorithm. If $c<i$, $\pi_c$ or its dominator must be in $\Pi^*$ when traversing $\pi_i$, and $\pi_i$ must be removed, which also contradicts the algorithm. Therefore the assumption is false.
    
    Summarizing these all concludes the proof.
\end{proof}

\begin{algorithm}[htbp]
\caption{Using the teacher to obtain non-convex Pareto frontier, based on insertion sort}
\label{alg:21}
\begin{algorithmic}[1]
\STATE Initialize the estimated Pareto frontier $\Pi^*=\emptyset$
\FOR{each policy $\pi_i \in \Pi$}
    \FOR{each policy $\pi_j \in \Pi^*$}
        \IF{$\boldsymbol{w}_k^T \boldsymbol{R}(\sigma_i)>\boldsymbol{w}_k^T \boldsymbol{R}(\sigma_j)$ for each $\boldsymbol{w}_k\in W_I$}
            \STATE Remove $\pi_j(a|s)$ from $\Pi^*$
        \ELSIF{$\boldsymbol{w}_k^T \boldsymbol{R}(\sigma_i)<\boldsymbol{w}_k^T \boldsymbol{R}(\sigma_j)$ for each $\boldsymbol{w}_k\in W_I$}
            \STATE break
        \ENDIF
    \ENDFOR
        \STATE Add $\pi_i$ into $\Pi^*$
\ENDFOR
\end{algorithmic}
\end{algorithm}

\section{A Different Perspective on Pb-MORL Motivation} \label{app:colored_glasses}

From a different perspective, our motivation for proposing Pb-MORL comes from the common issue of subjective bias in preference data, which exists in traditional PbRL and any application that relies on human expert preferences. 
For example, in autonomous driving, some people prefer aggressive driving that prioritizes efficiency, while others prefer safer, more cautious driving. 
Similarly, in large language models using RLHF (Reinforcement Learning from Human Feedback), the preferences assigned by different annotators often conflict due to personal biases. 
This subjectivity not only affects learning quality but also makes it difficult for models to accommodate diverse needs.

However, we believe that these scenarios can be reframed as multi-objective optimization problems. 
To address this, we propose Pb-MORL, which models the problem as a multi-objective optimization task. 
This approach learns a multi-objective policy that can effectively handle different weights $\boldsymbol{w}$, enabling better use of preference data, improving learning quality, and meeting the diverse needs of individuals.

}

\bibliographystyle{IEEEtran}
\bibliography{ref/refs_our_works, ref/refs_MORL, ref/refs_PbRL, ref/refs_rl, ref/refs_theory, ref/refs_misc, ref/refs_highway}

\end{document}